\documentclass[10pt,twoside]{article}

\RequirePackage[a4paper]{geometry}

\RequirePackage[T1]{fontenc}
\RequirePackage{lmodern}
\RequirePackage[final]{microtype}
  \RequirePackage{hyperref}
  
\RequirePackage[pagestyles]{titlesec}
\RequirePackage{titletoc}
\usepackage{graphicx}

\usepackage{xcolor}

\usepackage{amsmath,amssymb,amsthm}
\usepackage{mathrsfs}
\usepackage{physics}

\usepackage{numprint}
\npthousandsep{\,}

\usepackage{algorithm}
\usepackage[noend]{algpseudocode}
\usepackage[onehalfspacing]{setspace}

\usepackage{booktabs}

\usepackage{authblk}

\newtheorem{theorem}{Theorem}

\newtheorem{proposition}{Proposition}

\theoremstyle{definition}
\newtheorem{definition}{Definition}

\newcommand{\bangle}[1]{\langle #1 \rangle}

\newcommand{\half}{\frac{1}{2}}

\newcommand{\diff}{\mathop{}\!d}

\newcommand{\normal}{\mathcal N}
\newcommand{\uniform}{\mathbb U}
\newcommand{\discrete}{\text{Discr}}

\newcommand{\reals}{\mathbb R}
\newcommand{\naturals}{\mathbb N}

\newcommand{\expectation}{\mathbb E}

\newcommand{\sor}{\,\lor\,}
\newcommand{\sand}{\,\land\,}
\newcommand{\snot}{\lnot\,}
\newcommand{\until}{\ \mathbf U\ }
\newcommand{\eventually}{\ \mathbf F\ }
\newcommand{\always}{\ \mathbf G\ }
\newcommand{\true}{tt}

\newcommand{\kr}{k}

\title{A kernel function for Signal Temporal Logic formulae\thanks{This research has been partially supported by the Austrian FWF projects ZK-35 and by the Italian PRIN project ``SEDUCE'' n. 2017TWRCNB.}}

\author{
Luca Bortolussi$^{1,2}$, Giuseppe Maria Gallo$^1$,
and 
Laura Nenzi$^{1,3}$\\%
$^1$ Department of Mathematics and Geoscience, University of Trieste, Italy\\%
$^2$ Modelling and Simulation Group, Saarland University, Germany \\%
$^3$ University of Technology, Vienna, Austria
}

\date{}
\renewcommand\date[1]{}

\begin{document}


\maketitle
\begin{abstract}
We discuss how to define a kernel for Signal Temporal Logic (STL) formulae. Such a kernel allows us to embed the space of formulae into a Hilbert space, and opens up the use of kernel-based machine learning algorithms in the context of STL. We show an application of this idea to a regression problem in formula space for probabilistic models. 
\end{abstract}

\section{Introduction}
\label{sec:intro}

Signal Temporal Logic (STL)~\cite{Maler2004} is gaining momentum as a requirement specification language for complex systems and, in particular, Cyber-Physical Systems ~\cite{bartocci2018specification}.
STL has been applied in several flavours, from Runtime-monitoring~\cite{bartocci2018specification} to control synthesis~\cite{HaghighiMBB19} and falsification problems~\cite{FainekosH019}, and recently also within learning algorithms, trying to find a maximally discriminating formula between sets of trajectories~\cite{bombara_decision_2016, BBS14, NenziSBB18}. In these applications, a central role is played by the real-valued quantitative semantics
\cite{donze2013efficient}, measuring robustness of satisfaction.  

Most of the applications of STL have been applied to deterministic (hybrid) systems, with less emphasis on non-deterministic or stochastic ones~\cite{BartocciBNS15}. 
Another area in which formal methods are  providing interesting tools is in logic-based distances between models, like bisimulation metrics for Markov models~\cite{BacciBLM16}, which are typically based on a branching logic. In fact, extending these ideas to linear time logic is hard \cite{jan2016linear}, and typically requires statistical approximations. 
Finally, another relevant problem is how to measure the distance between two logic formulae, thus giving a metric structure to the formula space, a task relevant for learning which received little attention for STL, with the notable exception of \cite{madsen2018metrics}.

In this work, we tackle the metric,  learning, and model distance problems from a different perspective than the classical one, which is based on some form of  comparison of the languages of formulae. The starting point is to consider an STL formula as a function mapping a real-valued trajectory (signal) into a number or into another trajectory. As signals are functions, STL formulae should be properly considered as functionals, in the sense of Functional Analysis (FA) \cite{brezis2010functional}. This point of view gives us a large bag of FA tools to manipulate formulae. What we explore here is the definition of a suitable inner product in the form of a kernel \cite{shawe2004kernel} between STL formulae, capable of capturing the notion of semantic similarity of two formulae. This will endow the space of formulae with the structure of a Hilbert space,  defining a metric from the inner  product. Moreover, having a kernel opens the use of kernel-based machine learning techniques~\cite{rasmussen:williams:2006}.

A crucial aspect is that kernels for functionals are typically defined by integrating over the support space, with respect to a given measure. However, in trajectory space, there is no canonical measure (unless one discretizes time and maps signals to $\mathbb{R}^n$), which introduces a degree of freedom on which measure to use. We decide to work with probability measures on trajectories, i.e. stochastic processes, and we build one that favours 
``simple'' trajectories, with a small total variation.  This encodes the idea that  two formulae differing on simple signals should have a larger distance than two formulae differing only on complex trajectories. As we will see in the experiments, this choice allows the effective use of this kernel to perform regression on the formula space for approximating the satisfaction probability and the expected robustness of several stochastic processes, different than the one used to build the kernel. 

  \section{Background}
\paragraph{Signal Temporal Logic.}(STL)~\cite{Maler2004} is a linear time temporal logic suitable to monitor properties of continuous trajectories. 
A trajectory is a function $\xi: I\to D$ with $I$ a \textit{time domain} in $\reals_{\geq 0}$ and $D\subseteq\reals^n$ the \textit{state space}.
  We define the \textit{trajectory space}  $\mathcal T$ as the set of all possible continuous functions over $D$. The syntax of STL is: 
   	$$\varphi:=\true|\mu|\snot\varphi|\varphi_1\sand\varphi_2|\varphi_1\until_{[a, b]}\varphi_2$$
  	where $\true$ is the Boolean \textit{true} constant, $\mu$ is an atomic predicate, {\it negation} $\snot$ and {\it conjunction} $\sand$  are the standard Boolean connectives and $\until_{[a, b]}$ is the \textit{until} operator, with $a, b\in\reals$ and $a<b$.  As customary, we can derive the {\it disjunction} operator $\vee$ and the future  {\it eventually} $\eventually_{[t_{1},t_{2}]}$ and  {\it always} $\always_{[t_{1},t_{2}]}$ operators from  the until temporal modality. The logic has two semantics: a Boolean semantics, $\xi \models \varphi$, with the meaning that the trajectory $\xi$ satisfies the formula $\varphi$ and a quantitative semantics, $\rho(\varphi, \xi) $, that can be used to measure the quantitative level of satisfaction of a formula for a given trajectory. The function $\rho$ is also called the {\it robustness} function. The robustness is compatible with the Boolean semantics since it satisfies the soundness property:  if $\rho(\varphi, \xi, t) > 0$ then $(\xi,t)\models\varphi$; if $\rho(\varphi, \xi, t) < 0$ then $(\xi,t)\not\models\varphi$. Furthermore it satisfies also the correctness property, which shows that $\rho$ 	measures how robust is the satisfaction of a trajectory with respect to perturbations. We refer  the reader to~\cite{donze2013efficient} for more details.
  	

Given a \textit{stochastic process} $\mathcal M = (\mathcal T, \mathcal A, \mu)$, where $\mathcal T$ is a trajectory space  and $\mu$ is a probability measure on a $\sigma$-algebra $\mathcal A$ of $\mathcal T$, we define the  \textit{expected robustness} as $\bangle{\rho(\varphi, t)} := \expectation[\rho(\varphi, \xi, t)] = \int_{\xi\in\mathcal T}\rho(\varphi,\xi,t)d\mu(\xi)$. 
 The qualitative counterpart of the expected robustness is the \textit{satisfaction probability} $ S(\varphi, t)$, i.e. the probability that a trajectory generated by the stochastic process $\mathcal M$ satisfies the formula $\varphi$ at the time $t$: $ \expectation[s(\varphi, \xi, t)] = \int_{\xi\in\mathcal T}s(\varphi,\xi,t)d\mu(\xi)$
 where $s(\varphi,\xi,t) = 1$ if $(\xi, t)\models\varphi$ and $0$ otherwise. 
 The satisfaction probability $S(\varphi, t)$ is the probability that a trajectory generated by the stochastic process $\mathcal M$ satisfies the formula $\varphi$ at the time $t$.

\paragraph{Kernel Functions.} A \emph{kernel} $\kr(x, z)$, $x,z\in X$, defines an integral linear operator on functions $X\rightarrow\mathbb{R}$, which intuitively can be thought of as a scalar product on a possibly infinite feature space $F$:  $\kr(x, z) = \bangle{\phi(x), \phi(z)}$, with $\phi:X\rightarrow F$ being the eigenfunctions of the linear operator, spanning a Hilbert space, see~\cite{rasmussen:williams:2006}. 
Knowledge of the kernel allows us to perform approximation and learning tasks over $F$ without explicitly constructing it. 
%




One application is kernel regression, with the goal of estimating the function
$f(x) = \expectation_y[y|x]\,,\ x\in X$, from a finite amount of observations  $\{x_1, ..., x_N\}\subset X$, where each observation $x_i$ has an associated response $y_i\in \reals$, and $\{(x_1, y_1), ..., (x_N, y_N)\}$ is the \emph{training set}.
 There exist several methods that address this problem exploiting the kernel function $\kr:X\times X\to\reals$ has a similarity measure between a generic $x\in X$ and the observations $\{x_1, ..., x_N\}$ of the training set.
 In the experiments, we compare different regression models used to compute the expected robustness and the probability satisfaction.
    
\section{A kernel for Signal Temporal Logic}
\label{sec:kernelSTL}
If we endow an arbitrary space with a kernel function, we can apply different kinds of regression methods. 
Even for a non-metric space such as the STL formulae one, with a kernel we could perform operations that are very  expensive, such as the estimation of the 
 satisfaction probability and the expected robustness for a stochastic model of any formula $\varphi$, without running additional simulations.
 The idea behind our definition is to exploit the robustness to project an STL formula to a Hilbert space, and then to compute the scalar product in that space. In fact, the more similar the two projections will be, and the higher the scalar product will result. In addition, the function that we will define will be a kernel by construction.
  
\subsection{STL kernel}
\label{par:kernel}
  Let us fix a formula $\varphi\in\mathcal{P}$ in the STL formulae space. Consider the robustness $\rho(\varphi, {}\cdot{}, {}\cdot{}):\mathcal{T}\times I\to\reals$, $I\subset\reals$ is a bounded interval, and  $\mathcal{T}$ is the trajectory space of continuous functions. We observe that there is a map $h: \mathcal{P}\to C(\mathcal{T}\times I)$ defined by $h(\varphi)(w, t) = \rho(\varphi, w, t)$. With $C(X)$ we denote the set of the continuous functions on the topological space $X$. It can be proved that
  $h(\mathcal P)\subseteq L^2(\mathcal{T}\times I)$ 
    and hence we can use the dot product in $L^2$ as a kernel for $\mathcal{P}$. Formally,
\begin{theorem}
Given the STL formulae space  $\mathcal{P}$, the trajectory space $\mathcal{T}$, a bounded interval $I\subset\reals$, let  $h: \mathcal{P}\to C(\mathcal{T}\times I)$ defined by $h(\varphi)(w, t) = \rho(\varphi, w, t)$, then:
   \begin{equation}\label{subset L}
  h(\mathcal P)\subseteq L^2(\mathcal{T}\times I)\,
  \end{equation} 
  \label{th:L2}
\end{theorem}

For proving the theorem, we need to recall the definition of $L^2$ and its inner product.
  \begin{definition}
  	Given a measure space $(\Omega, \mu)$, we call Lebesgue space $L^2$ the space defined by
  	$$L^2(\Omega) = \{f\in\Omega: \|f\|_{L^2} < \infty\}\,,$$
  	where $\|\cdot\|_{L^2} $ is a norm defined by
  	$$\|f\|_{L^2} =\left(\int_\Omega |f|^2 d\mu\right)^{\half}\,.$$
  	We define the function $\bangle{\cdot,\cdot}_{L^2}:L^2(\Omega)\times L^2(\Omega)\to\reals$ as
  	$$\bangle{f,g}_{L^2} = \int_\Omega fg\,.$$
  	It can be easily proved that $\bangle{\cdot, \cdot}_{L^2}$ is a inner product.
  \end{definition}
  Furthermore, we have the following result.
  \begin{proposition}[\cite{kothe1983topological}]
  	$L^2(\Omega)$ with the inner product $\bangle{\cdot, \cdot}_{L^2}$ is a Hilbert space.
  \end{proposition}
  
\begin{proof}[Proof Theorem \ref{th:L2}]
  In order to satisfy (\ref{subset L}) we can make the hypothesis that
  $\mathcal T$ is a bounded (in the $\sup$ norm) subset of $C(I)$, with $I$ a bounded interval, which means that exists $M\in\reals$ such that $\|\xi\|_{\infty} = \sup_{t\in I} \xi(t) \leq \infty$ for all $\xi\in\mathcal T$. Moreover, the measure on $\mathcal T$ is a distribution, and so it is a finite measure. Hence
  \begin{align*}
  \int_{\xi\in\mathcal T}\int_{t\in I}|h(\varphi)(\xi, t)|^2dtd\mu = &\int_{\xi\in\mathcal T}\int_{t\in I}|\rho(\varphi, \xi, t)|^2dtd\mu \\\leq &\int_{\xi\in\mathcal T}\int_{t\in I}|B + M(\varphi)|^2dtd\mu \\\leq &(B+M(\varphi))^2|I|
  \end{align*}
  for each $\varphi\in\mathcal P$ and $M(\varphi)$ is the maximum absolute value of an atomic proposition of $\varphi$. This implies $h(\mathcal P)\subseteq L^2(\mathcal T\times I)$.
\end{proof} 
  
We can now use the dot product in $L^2$ as a kernel for $\mathcal{P}$. In such a way, we will obtain a kernel that returns a high positive value for formulae that agree on high-probability trajectories and high negative values for formulae that, on average, disagree. 

\begin{definition}
    Fixing a probability measure $\mu_0$ on $\mathcal{T}$, we can then define the STL-kernel as:
  \begin{equation*}
  	\kr'(\varphi, \psi) = \bangle{h(\varphi), h(\psi)}_{L^2(\mathcal T\times I)} = \int_{\xi\in\mathcal T}\int_{t\in I}h(\varphi)(\xi, t)h(\psi)(\xi, t)dtd\mu_0 = \int_{\xi\in\mathcal T}\int_{t\in I}\rho(\varphi, \xi, t)\rho(\psi, \xi, t)dtd\mu_0\,
  	\end{equation*}
 \end{definition}

 Since the function $\kr'$ satisfies the finitely positive semi-definite property, we can be proved that is a kernel itself.

\begin{proposition}\label{kernel matrices characterization}
  	Kernel matrices are positive semi-definite.
  \end{proposition}
  \begin{proof}
  	Let us consider the general case, that is $G_{ij} = \bangle{\phi(x_i), \phi(x_j)}$ for $i, j = 1, ..., k$. Let us consider
  	\begin{align*}
  	v^T Gv &= \sum_{i, j=1}^k v_i v_j G_{ij}= \sum_{i, j=1}^k v_i v_j \bangle{\phi(x_i), \phi(x_j)}\\
  	&=  \left\langle\sum_{i=1}^kv_i\phi(x_i),\sum_{j=1}^kv_j\phi(x_j)\right\rangle \\
  	&=\left\|\sum_{i=1}^kv_i\phi(x_i)\right\|^2\geq 0\,,
  	\end{align*}
  	which implies that $G$ is positive semi-definite.
  \end{proof}

\begin{theorem}[Characterization of kernels]\label{kernels characterization}
  	A function $\kr:X\times X\to\reals$ which is either continuous or has a finite domain, can be written as
  	$$\kr(x,z)=\bangle{\phi(x),\phi(z)}\,,$$
  	where $\phi$ is a feature map into a Hilbert space $F_\kr$, if and only if it satisfies the finitely positive semi-definite property.
  \end{theorem}
  \begin{proof}
  	Firstly, let us observe that if $\kr(x,z)=\bangle{\phi(x),\phi(z)}$, then it satisfies the finitely positive semi-definite property for the Proposition \ref{kernel matrices characterization}. The difficult part to prove is the other implication.
  	
  	Let us suppose that $\kr$ satisfies the finitely positive semi-definite property. We will construct the Hilbert space $F_\kr$ as a function space. We recall that $F_\kr$ is a Hilbert space if it is a vector space with an inner product that induces a norm that makes the space complete.
  	
  	Let us consider the function space
  	$$\mathcal F=\left\{\sum_{i=1}^n\alpha_i\kr(x_i,\cdot)\,:\,n\in\naturals,\ \alpha_i\in\reals,\ x_i\in X,\ i=1,...,n\right\}\,.$$
  	The sum in this space is defined as
  	$$(f+g)(x)=f(x)+g(x)\,,$$
  	which is clearly a close operation. The multiplication by a scalar is a close operation too. Hence, $\mathcal F$ is a vector space.
  	
  	We define the inner product in $\mathcal F$ as follows. Let $f, g\in\mathcal F$ be defined by
  	$$f(x) = \sum_{i=1}^n\alpha_i\kr(x_i,x)\,, g(x) = \sum_{i=1}^m\beta_i\kr(z_i,x)\,,$$
  	so the inner product is defined as
  	$$\bangle{f,g}:=\sum_{i=1}^n\sum_{j=1}^m\alpha_i\beta_j\kr(x_i,z_j) = \sum_{i=1}^n\alpha_i g(x_i)=\sum_{j=1}^m\beta_i f(z_j)\,,$$
  	where the last two equations follows from the definition of $f$ and $g$.
  	This map is clearly symmetric and bilinear. So, in order to be an inner product, it suffices to prove
  	$$\bangle{f,f}\geq 0\text{ for all } f\in\mathcal F\,,$$
  	and that
  	$$\bangle{f,f} = 0 \iff f \equiv 0\,.$$
  	If we define the vector $\alpha = (\alpha_1,...,\alpha_n)$ we obtain
  	$$\bangle{f,f}=\sum_{i=1}^n\sum_{j=1}^n\alpha_i\alpha_j\kr(x_i,x_j)=\alpha^T K\alpha\geq 0\,,$$
  	where $K$ is the kernel matrix constructed over $x_1,...,x_n$ and the last equality holds because $\kr$ satisfies the finite positive semi-definite property.
  	
  	It is worth to notice that this inner product satisfies the property
  	$$\bangle{f,\kr(x,\cdot)} = \sum_{i=1}^n\alpha_i \kr(x_i,x)=f(x)\,.$$
  	This property is called \emph{reproducing property} of the kernel.
  	
  	From this property it follows also that, if $\bangle{f,f}=0$ then
  	$$f(x) = \bangle{f,\kr(x,\cdot)}\leq \|f\|\kr(x,x) = 0\,,$$
  	applying the Cauchy-Schwarz inequality and the definition of the norm. The other side of the implication, i.e. $$f\equiv 0\implies\bangle{f,f} =0\,,$$ follows directly from the definition of the inner product.
  	
  	It remains to show the completeness property. Actually, we will not show that $\mathcal F$ is complete, but we will use $\mathcal F$ to construct the space $F_\kr$ of the enunciate. Let us fix $x$ and consider a Cauchy sequence $\{f_n\}_{n=1}^\infty$. Using the reproducing property we obtain
  	$$(f_n(x)-f_m(x))^2=\bangle{f_n-f_m,\kr(x,\cdot)}^2\leq\|f_n-f_m\|^2\kr(x,x)\,.$$
  	where we applied the Cauchy-Schwarz inequality. So, for the completeness of $\reals$, $f_n(x)$ has a limit, that we call $g(x)$. Hence we define $g$ as the punctual limit of $f_n$ and we define $F_\kr$ as the space obtained by the union of $\mathcal F$ and the limit of all the Cauchy sequence in $\mathcal F$, i.e.
  	$$F_\kr = \overline{\mathcal F}\,,$$
  	which is the closure of $\mathcal F$. Moreover, the inner product in $\mathcal F$ extends naturally in an inner product in $F_\kr$ which satisfies all the desired properties.
  	
  	In order to complete the proof we have to define a map $\phi:X\to F_\kr$ such that
  	$$\bangle{\phi(x),\phi(z)}=\kr(x,z)\,.$$
  	The map $\phi$ that we are looking for is $\phi(x) = \kr(x,\cdot)$. In fact
  	$$\bangle{\phi(x),\phi(z)}=\bangle{\kr(x,\cdot),\kr(z,\cdot)}=\kr(x,z)\,.$$
  \end{proof}

  One desirable property of our kernel is that $
  \kr(\varphi, \varphi) \geq \kr(\varphi, \psi)\,,\ \forall \varphi, \psi\in\mathcal P.$   In fact, given a formula $\varphi$, no formula should be more similar to $\varphi$ then $\varphi$ itself.  This property can be enforced by redefining the kernel as follows:
  \begin{equation*}\label{normalised kernel}
  \kr(\varphi, \psi) = \frac{\kr'(\varphi, \psi)}{\sqrt{\kr'(\varphi, \varphi)\kr'(\psi, \psi)}}\,.
  \end{equation*}

\subsection{The base measure $\mu_0$}
\label{par:trajectory space.}
  In order to make our kernel meaningful and not too expensive to compute, we endow the trajectory space $\mathcal T$ with a probability distribution such that more complex trajectories are less probable.  We use the total variation~\cite{pallara2000functions} of a trajectory and the number of changes in its monotonicity  as indicators of its "complexity". We define the probability measure $\mu_0$ by providing an algorithm sampling from piece-wise linear functions, a dense subset of $\mathcal{T}$, that we use for Monte Carlo approximation of $\kr$.
  
   Before describing the algorithm, we need the definition of Total Variation of a function \cite{pallara2000functions}.
  \begin{definition}[Total Variation]
  	Let $f\in C(I)$. We call \emph{Total Variation} of $f$ on a finite interval $[a, b]$ the quantity
  	\begin{equation}\label{bv norm}
  	V_a^b(f) = \sup_{P\in \mathbf P}\sum_{i=0}^{n_{P}-1}|f(x_{i+1})-f(x_i)|\,,\ \forall f\in C(I)\,
  	\end{equation}
  	where $\mathbf P$ is the set of all partitions of the interval $[a, b]$.
  \end{definition}
  We use the total variation of a trajectory as an indicator of its "complexity".
  We also take the number of changes in the monotonicity behavior of a trajectory as another indicator of "complexity". The idea is to endow $\mathcal T$ with a probability distribution such that more complex trajectories are less likely to be drawn. We describe a sampling algorithm over piecewise linear functions that we use for Monte Carlo approximation. In doing so, we sample from a dense subset of $C(I)$. 
  
  The sampling algorithm is the following:
  \begin{enumerate}
  	\item Set a discretization step $h$; define $N = \frac{b-a}{h}$ and $t_i = a + ih$;
  	\item Sample a starting point $\xi_0\sim\normal(0, \sigma')$ and set $\xi(t_0) = \xi_0$;
  	\item Sample $K\sim(\normal(0, \sigma''))^2$, that will be the total variation of $\xi$;
  	\item Sample $N-1$ points $y_1,...,y_{N-1}\sim\uniform([0, K])$ and set $y_0=0$ and $y_n=K$;
  	\item Order $y_1, ..., y_{N-1}$ and rename them such that $y_1 \leq y_2 \leq...\leq y_{N-1}$;
  	\item Samle $s_0\sim\discrete(-1, 1)$;
  	\item Set iteratively $\xi(t_{i+1}) = \xi(t_i) + s_{i+1}(y_{i+1}-y_i)$ with $s_{i+1} = s_is$,\\
  	$P(s=-1) = q$ and $P(s=1) = 1-q$, for $i = 1, 2, ..., N$.
  \end{enumerate}
  Finally, we can linearly interpolate between consecutive points of the discretization and make the trajectory continuous, i.e., $\xi\in C(I)$.
  
  \begin{figure}[!t]
  	\begin{center}
  		\includegraphics[scale=0.8]{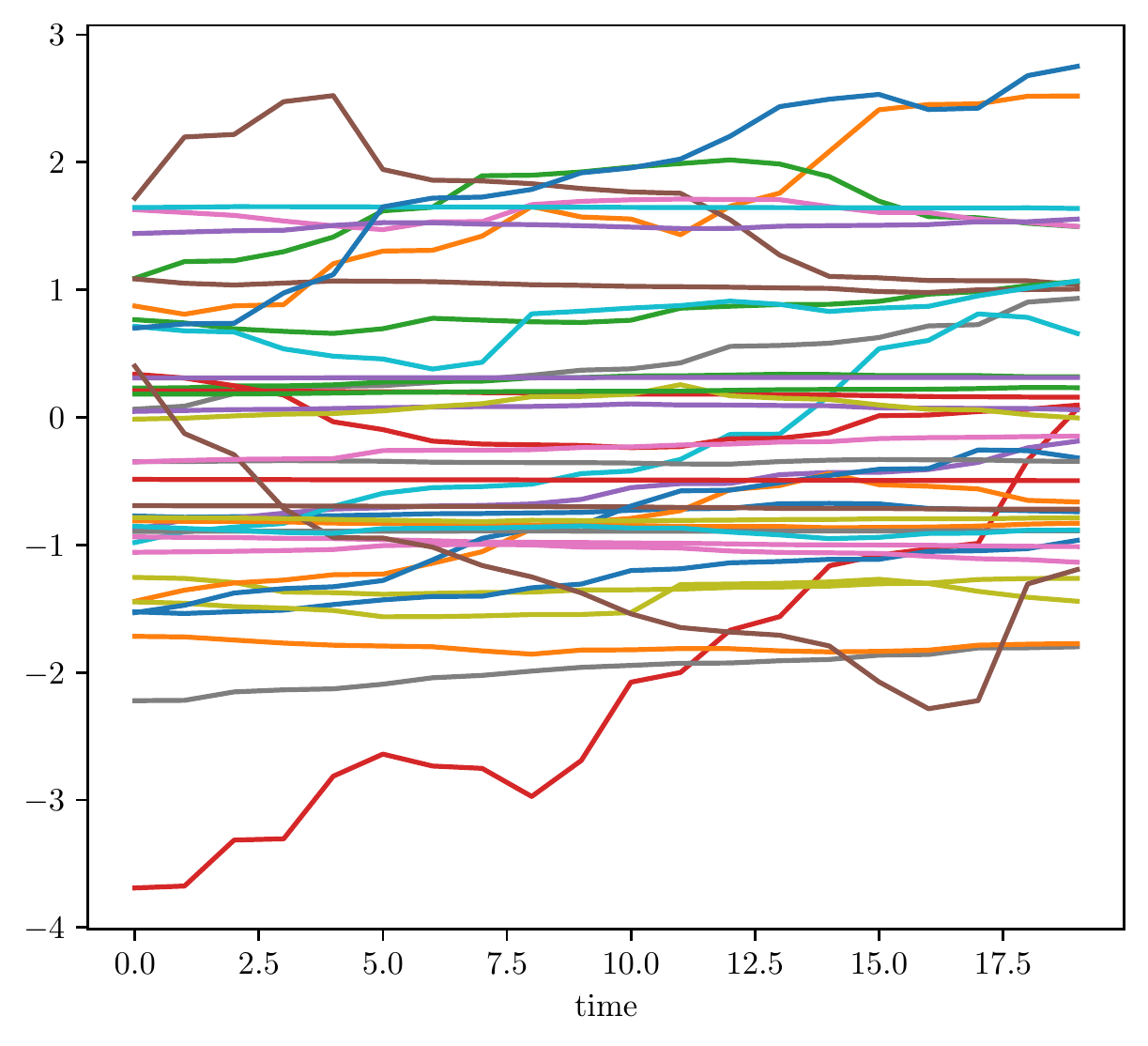}
  	\end{center}
  	\caption{\label{fig:trajectories_base}Trajectories randomly sampled from $\mathcal T$ using the above descripted algorithm.}
  \end{figure}
  For our implementation, we fixed the above parameters as follows:
  \begin{itemize}
  	\item $a = 0$,
  	\item $b = 20$,
  	\item $h = 1$,
  	\item $\sigma' = 1$,
  	\item $\sigma'' = 1$,
  	\item $q = 0.1$.
  \end{itemize}

  In the next section, we show that using this simple measure still allows us to make predictions with remarkable accuracy for other stochastic processes on $\mathcal{T}$.

\section{Experimental Results}
\label{sec:exp}
\subsection{Kernel Regression on $\mu_0$} 
To show the goodness of our kernel definition, we use it to predict the expected robustness and the satisfaction probability of STL formulae w.r.t. the stochastic process $\mu_0$ defined on $\mathcal{T}$. 
We use a training set composed of 400 formulae sampled randomly according to a syntax tree random growing scheme as follows:
  \begin{enumerate}
  	\item Sample the number of atomic predicates $$n\sim\discrete(1,2,...,6)\,;$$
  	\item Sample $$k_i\sim\uniform([-7, 7])\,;$$
  	\item Create a set formulae $$P = \{\mu_1,...,\mu_n\}\,,$$ where $\mu_i$ is the atomic predicate $(x \geq k_i)$;
  	\item With $50\%$ of probability select the operator $*=\snot\,$; with the remaining $50\%$ sample an operator $$*\sim\discrete(\sor,\sand,\until,\eventually,\always)\,;$$
  	\item Randomly sample 1 or 2 formulae from $P$ (depending if $*$ is an unary or a binary operator) and apply $*$ to it or them, obtaining the formula $\varphi$;
  	\item Remove the formula/formulae sampled at step (5) from $P$ and add $\varphi$ to $P$;
  	\item If $P$ has more than one element, repeat from step (4), otherwise continue to step (8);
  	\item The output formula is, with $50\%$ of probability, the last formula of $P$; for the other $50\%$ of probability, sample an operator $$*\sim\discrete(\snot,\eventually,\always)$$ and apply it to the last formula of $P$: the resulting formula is the output of the algorithm.
  \end{enumerate}

Then, we approximate expected robustness and satisfaction probability using a set of \numprint{100000} trajectories sampled according to $\mu_0$.
We compare the following regression models: {\it Nadaraya-Watson estimator}, {\it K-Nearest Neighbors regression}, {\it Support Vector Regression} (SVR) and {\it Kernel Ridge Regression} (KRR) \cite{murphy2012machine}. We obtain the lowest \emph{Mean Squared Error} (MSE) on expected robustness, equal to $0.29$, using an SVR with a Gaussian kernel and $\sigma=0.5$. On the other hand, the best performances in predicting the satisfaction probability were given by the KRR, with an MSE equal to $0.00036$.

\paragraph{Kernel Regression on  other stochastic processes} 
The last aspect that we investigate is whether the definition of our kernel w.r.t. the fixed measure $\mu_0$ can be used for making predictions of the average robustness also 
for other stochastic processes, i.e., while taking expectations w.r.t. other probability measures $\mu$ on $\mathcal{T}$. We compare this with a kernel defined w.r.t $\mu$ itself. 
We used three different stochastic models: \emph{Immigration}, \emph{Isomerization}  and \emph{Polymerase}, simulated using the Python library StochPy \cite{maarleveld2013stochpy}, Figure~\ref{fig:stoch_trajectories}.   
\begin{figure}[!t]
  	\begin{center}
  		\hspace*{-0cm}\includegraphics[scale=0.22]{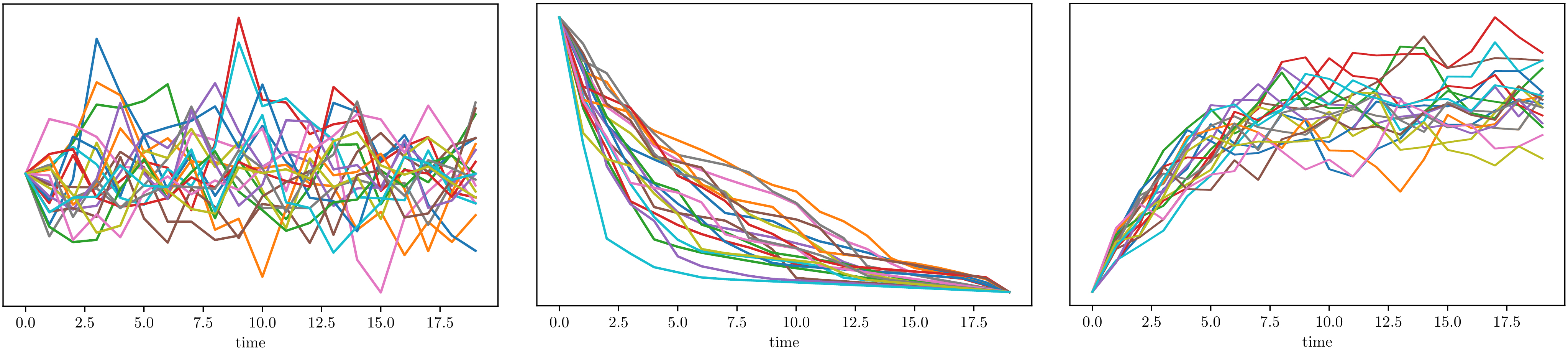}
  	\end{center}
  	\vspace*{-0cm}\caption{\label{fig:stoch_trajectories}From left to right, trajectories generated by the \emph{Immigration} model, by the \emph{Isomerization} model and by the \emph{Polymerase} model.}
  \end{figure}

As it can be seen from Figure \ref{fig:compare_sigma_exp_rob_other_stoch} (left), our base kernel is the best performing one. 
This can be explained by the fact that the measure $\mu_0$ is broad in terms of coverage of the trajectory space, meaning that different kinds of behaviours tend to be covered. This allows to better distinguish among STL formulae, compared to models that tend to focus the probability mass on narrower regions of $\mathcal{T}$, such as the \emph{Isomerization} model (which has the worst performance). 
Also in this case, we obtained the best results using SVR and KRR. However, given the sparseness of SVR, it's more convenient to use it, since we need to evaluate a lower number of kernels to perform the regression. 
Interestingly, the minimum MSE is obtained using the Gaussian kernel with exactly the same $\sigma$ parameter as for the regression task on $\mu_0$, hinting for some intrinsic robustness to hyperparameter's choice that has to be investigated in greater detail. 
  In Figure \ref{fig:compare_sigma_exp_rob_other_stoch} (right) we show the predictions for the expected robustness over the three stochastic models that we took as examples, using the best regression model that we have found so far, which is the SVR with a Gaussian kernel having $\sigma = 0.22$. 
 Note that to compute the kernel by Monte Carlo approximation, we have to sample only once the required trajectories for $\mu_0$. We also need to estimate the expected robustness transition probability for the formulae comprising the training set. However, kernel regression permits us to avoid further simulations of the model $\mu$ for novel formulae $\phi$.  
 
 
  \begin{figure}[!t]
  	\begin{center}
  		\hspace*{-0cm}\includegraphics[scale=.41]{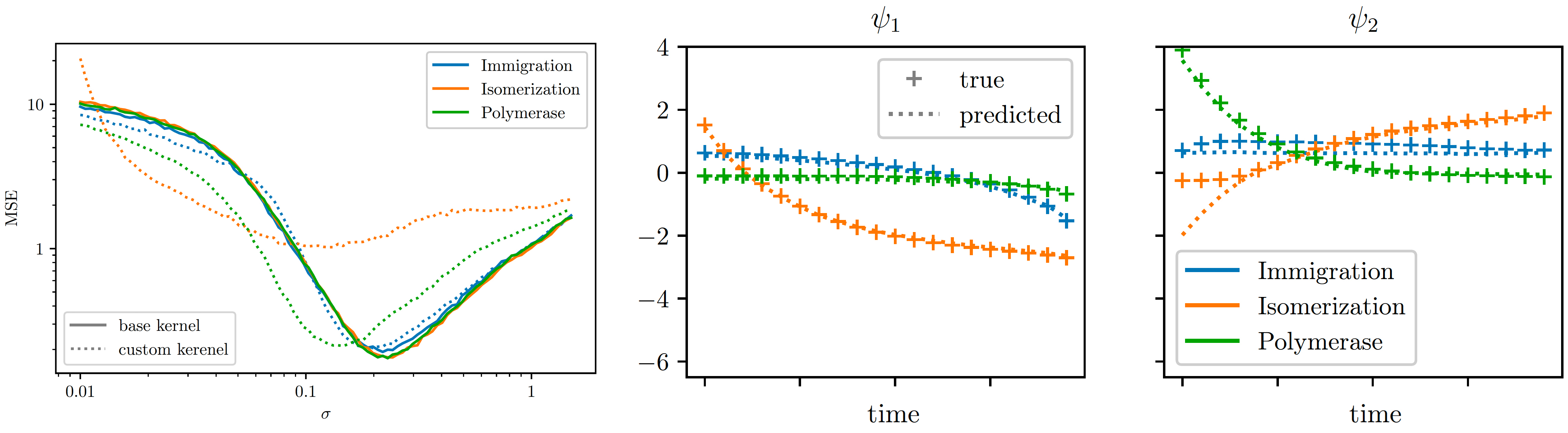}
  	\end{center}
  	\vspace*{-0cm}\caption{\label{fig:compare_sigma_exp_rob_other_stoch}({\bf left}) MSE as a function of the bandwidth $\sigma$ of the Gaussian kernel, for the prediction of the expected robustness. We compare the performances on different stochastic models, using both the kernel evaluated according to the base measure $\mu_0$ (base kernel), and a custom kernel computed using the trajectories generated by the stochastic model itself. Both the axis are in logarithmic scale. ({\bf center, right}) predictions of the expected robustness for formulae $\psi_1= \eventually x \geq 1.5$ and $\psi_2=(\snot(((\always x \geq -1.4) \sor (\always x \geq 2.7)) \sor x \leq 0.7)) \sor x \leq 0$, over three different trajectory spaces. The predictions are made using SVR on a Gaussian kernel, with the best performing bandwidth $\sigma$, which is $\sigma = 0.22$.
  	}
  \end{figure}
 
\section{Conclusions}

We defined a kernel for STL, fixing a base measure over trajectories, and we showed that we can use \emph{exactly} the same kernel across different stochastic models for computing a very precise approximation of the expected robustness of new formulae, with only the knowledge of the expected robustness of a fixed set of training formulae. 
Our STL-kernel, however, can also be used for other tasks. For instance, computing STL-based distances among stochastic models, resorting to a dual kernel construction, and building non-linear embeddings of formulae into finite dimensional real spaces with kernel-PCA techniques. Another direction for future work is to refine the quantitative semantics in such a way that equivalent formulae have the same robustness, e.g. using ideas like in \cite{madsen2018metrics}.


\bibliographystyle{abbrv}
\bibliography{biblio}

\clearpage

\end{document}